\documentclass[10pt,journal,compsoc]{IEEEtran}

\usepackage{cite}

\usepackage{amsmath,amssymb,amsfonts}
\usepackage{amsthm}
\usepackage{algorithm}
\usepackage{algpseudocode}
\usepackage{array}
\usepackage[caption=false,font=footnotesize,labelfont=sf,textfont=sf]{subfig}
\usepackage{dblfloatfix}
\usepackage{url}
\usepackage{balance}
\usepackage{todonotes}
\usepackage{hyperref}
\usepackage{mathtools}
\usepackage{multirow}
\usepackage{booktabs}
\hypersetup{
	colorlinks=true, 
	linktoc=all,     
	linkcolor={black},
	urlcolor={blue},
	anchorcolor=black, 
	citecolor=black,
}
\DeclareMathSymbol{:}{\mathord}{operators}{"3A}

\newtheorem{definition}{Definition}
\newtheorem{lemma}{Lemma}
\newtheorem{theorem}{Theorem}
\newcommand{\theMC}{(MC)$^\text{2}$MKP}
\newcommand{\Natural}{\mathbb{N}}
\newcommand{\Real}{\mathbb{R}_{\ge 0}}

\newcommand{\SetRes}{\mathcal{R}}
\newcommand{\SetCosts}{\mathcal{C}}
\newcommand{\SetLower}{\mathcal{L}}
\newcommand{\SetUpper}{\mathcal{U}}

\newcommand{\NumRes}{n}
\newcommand{\NumTasks}{T}
\newcommand{\Cost}[1]{C_{{#1}}}
\newcommand{\Mar}[1]{M_{{#1}}}
\newcommand{\Upper}[1]{U_{#1}}
\newcommand{\Lower}[1]{L_{#1}}
\newcommand{\Mapping}{X}
\newcommand{\Map}[1]{x_{{#1}}}
\newcommand{\TotalCost}{\Sigma C}
\DeclareMathOperator*{\argmin}{arg\,min}

\begin{document}
\title{Scheduling Algorithms for Federated Learning with Minimal Energy Consumption}

\author{La\'ercio~Lima~Pilla\\Univ. Bordeaux, CNRS, Bordeaux INP, Inria, LaBRI, UMR~5800, F-33400 Talence, France.\protect\\
E-mail: laercio.lima-pilla@labri.fr \IEEEcompsocitemizethanks{\IEEEcompsocthanksitem This work has been submitted to the IEEE for possible publication. Copyright may be transferred without notice, after which this version may no longer be accessible.}}

\maketitle

\begin{abstract}
Federated Learning (FL) has opened the opportunity for collaboratively training machine learning models on heterogeneous mobile or Edge devices while keeping local data private.
With an increase in its adoption, a growing concern is related to its economic and environmental cost (as is also the case for other machine learning techniques).
Unfortunately, little work has been done to optimize its energy consumption or emissions of carbon dioxide or equivalents, as energy minimization is usually left as a secondary objective.
In this paper, we investigate the problem of minimizing the energy consumption of FL training on heterogeneous devices by controlling the workload distribution.
We model this as the Minimal Cost FL Schedule problem, a total cost minimization problem with identical, independent, and atomic tasks that have to be assigned to heterogeneous resources with arbitrary cost functions.
We propose a pseudo-polynomial optimal solution to the problem based on the previously unexplored Multiple-Choice Minimum-Cost Maximal Knapsack Packing Problem.
We also provide four algorithms for scenarios where cost functions are monotonically increasing and follow the same behavior.
These solutions are likewise applicable on the minimization of other kinds of costs, and in other one-dimensional data partition problems. 

\end{abstract}

\section{Introduction}\label{sec:intro}

Federated Learning~(FL) is a distributed machine learning technique used for training a shared model collaboratively while not sharing local data~\cite{mcmahan2017, bonawitz2019, lim2020, gu2021}.
This technique reduces privacy and security risks while also improving communication efficiency~\cite{mcmahan2017}.
These features make FL attractive for applications from next-word prediction~\cite{hard2018} and on-device item ranking~\cite{bonawitz2019}, to cyberattack detection~\cite{lim2020} and graph classification~\cite{xie2021}.
Due to its data privacy design, FL has also received significant attention in medical applications. It has been used in brain tumor segmentation~\cite{sheller2019}, tumor classification~\cite{brum2021}, and chest X-ray diagnosis for COVID-19~\cite{park2021}.

In its most standard form, FL is based on the idea of a central server that coordinates the work of participating devices (mostly heterogeneous mobile or edge devices, but also local computers or Cloud instances in some application cases).
The server starts a training round by sending an initial model to some of the devices.
The devices train the model with their own local data, and send the updated model weights back to the server.
The latter then aggregates all updates and combines their model weights (e.g., by averaging their values) in order to start the next training round.
This process is repeated until a given deadline is met, a fixed number of training rounds is achieved, or until the model converges to a target accuracy.


The accuracy of \textbf{machine learning models (FL included)} has seen improvements at the cost of larger models and \textbf{exponentially-growing computational demands}.
When combined with an increase in usage, this is leading to an increased attention to machine learning’s economic and environmental costs~\cite{schwartz2020,henderson2020}.
An initial study~\cite{qiu2021} indicates that \textbf{FL’s energy consumption can be one order of magnitude greater than an equivalent centralized model, while its carbon footprint may even be two orders of magnitude larger} (mostly due to the energy sources available for people participating in training in different locations around the globe).
Energy is also of concern for FL due to the limited batteries of mobile devices.


All of the aforementioned concerns make FL a prime target for energy consumption optimizations.
Yet, little work has been done on the subject so far~\cite{kim2021}.
Energy consumption is often seen as a secondary or tertiary objective after accuracy and execution time~\cite{xu2019,li2019,tran2019,nguyen2021}.
Additionally, the energy consumed during training is usually modeled in a simple manner to reduce the complexity of the optimization problems and ease the use of heuristics~\cite{anh2019,kang2019,li2021,nguyen2021,tran2019,zaw2021,yang2021}.


In this context, \textbf{we present optimal scheduling algorithms to minimize the energy consumption of Federated Learning}.
Our algorithms define, for a target volume of training data, how much local data each device should use (while also respecting their own lower and upper data limits).
We show that this scheduling problem, in its general definition (without any specific assumptions of the devices’ energy behavior), is equivalent to a \textbf{new knapsack problem that can be solved in pseudo-polynomial time}.
We also present \textbf{four algorithms with lower complexity for specific energy behavior scenarios} where costs increase monotonically.
All of these algorithms have been designed for energy conservation, but they can be directly applied to minimize the carbon footprint, monetary cost, or any other cost function, weighted or not.
We can summarize our main contributions as follows:

\begin{itemize}
\item[{$\star$}] We provide a formulation of the minimal cost FL schedule problem with lower limits, upper limits, and arbitrary cost functions per device;
\item[{$\star$}] We propose an optimal pseudo-polynomial solution based on the new Multiple-Choice Minimum-Cost Maximal Knapsack Packing Problem; and
\item[{$\star$}] We present four optimal scheduling algorithms for scenarios with monotonically increasing cost functions with specific marginal cost behaviors, with or without upper limits.
\end{itemize}

The remaining of this paper is organized as follows:
Section~\ref{sec:rw} introduces related work on FL, energy, and scheduling.
Section~\ref{sec:problem} provides a formulation of our scheduling problem.
Section~\ref{sec:general} presents an optimal solution to this problem based on a new knapsack problem and solution.
Section~\ref{sec:opts} shows algorithms optimized for specific scenarios.
Section~\ref{sec:conclusion} presents concluding remarks.

\section{Related Work}\label{sec:rw}

Federated Learning is the subject of a large body of work in both research and development. 
We point our readers to the surveys done by Lim et al.~\cite{lim2020}, Zhang et al.~\cite{zhang2021}, and Gu et al.~\cite{gu2021} for a comprehensive view on the subject and associated topics. 
Our focus here is dedicated to works that deal with topics related to
workload distribution,
energy optimization for FL,
and energy profiling and modeling.

\subsection{Workload Distribution}

When FL was originally proposed by McMahan et al.~\cite{mcmahan2017}, no special attention was paid to the training time or to the energy consumption of the mobile devices.
It is only natural then that no mechanism to control the workload distribution (i.e., how much local data each device should use) was proposed at the time.

Wang, Wei, and Zhou~\cite{wang2020} propose Fed-LBAP and Fed-MinAvg as mechanisms to control the workload distribution in order to minimize the computation time and accuracy loss when training FL models.
They identified the computation time as the main bottleneck for FL.
This argument has also been emphasized by others that postulate a decrease on the impact of communication during training with the rise of 5G technologies~\cite{gu2021,li2021}.
Wang, Yang, and Zhou~\cite{wang2021} have later presented an algorithm named MinCost for this same optimization problem.
We have previously proposed OLAR~\cite{pilla2021} as an optimal greedy solution for minimizing the computation time of FL.
Nonetheless, all of these algorithms help minimize the maximum cost (execution time, in this case), while our energy consumption problem requires the minimization of the total cost (its sum).

Outside FL, Khaleghzadeh, Manumachu, and Lastovetsky~\cite{khaleghzadeh2018} introduce a branch-and-bound solution to minimize the computation time when the cost function of each resource (how much time it takes to process a given amount of data) is arbitrary --- in other words, they do not follow any specific behavior like being monotonically increasing.
Their algorithm has its worst-case complexity in $O(\NumRes^3\NumTasks^3)$ for $\NumRes$ heterogeneous resources and a workload of size $\NumTasks$.
Later, Khaleghzadeh et al.~\cite{khaleghzadeh2021} have adapted this algorithm to optimize both execution time and energy consumption (i.e., find the Pareto front). 
Their new solution has its time complexity in~$O(\NumRes^3\NumTasks^3\log(\NumRes\NumTasks))$.
Meanwhile, our focus lies only on energy consumption.
As we will show later, our solution for arbitrary cost functions has a time complexity in~$O(\NumTasks^2\NumRes)$, while the solutions for specific scenarios vary between $O(\NumTasks\NumRes^2)$ and $\Theta(\NumRes)$. 

A constraint in our workload distribution comes in the form of lower and upper limits on the amount of work to be given to each device.
These limits play an important role in our problem.
For FL, workload is translated into the number of data samples (mini-batches) to be used for training during a round.
Lower limits can be used to enforce device participation at minimum levels (providing different data sources), which could also help with fairness~\cite{huang2021}.
They also help reduce the time a scheduling algorithm takes to achieve its decisions~\cite{pilla2021}.
Meanwhile, upper limits avoid data over-representation from better-performing or more energy-efficient devices~\cite{lim2020}.
They can be naturally found by considering the amount of data available in a device~\cite{anh2019}.
Finally, they can also be used to set contracts with participants and help incentivise their participation in training~\cite{kang2019}.

\subsection{Energy Optimization in FL}

Instead of controlling the workload distribution, most works that consider the energy consumption of FL devices act on the clock frequency of their processors~\cite{xu2019,li2019,tran2019,nguyen2021,yang2021,khan2020}, among other options.
For instance, Xu, Li, and Zou~\cite{xu2019} reduce the clock frequency of participating devices while respecting a set training round deadline in order to conserve energy. 
The same kind of strategy is employed by SmartPC~\cite{li2019}.
Tran et al.~\cite{tran2019} optimize for both time and energy by controlling the clock frequency and the fraction of communication time allocated to the devices.
Meanwhile, Khan et al.~\cite{khan2020} let Edge devices choose their own clock frequencies by using rewards as an incentive mechanism in a Stackelberg game.

The exploration of multiple optimization options concurrently is not uncommon in the literature. 
Nguyen et al.~\cite{nguyen2021} optimize both time and energy consumption by setting the transmission time, bandwidth allocation, and clock frequency of devices.
Yang et al.~\cite{yang2021} consider in addition the transmission power and learning accuracy when minimizing the total energy consumption of FL over wireless networks.
In contrast, we prefer to act only on the workload distribution as this keeps the optimization at the software level (i.e., no changes take effect in the devices’ hardware).
Still, acting on additional control points could be explored in future work.

There are also some works that do not use the clock frequency of the devices in their decisions, indicating other venues for optimization.
Luo et al.~\cite{luo2021} optimize time and energy together by choosing which devices should participate in a training round and by setting their local number of epochs to compute.
Li et al.~\cite{li2021} control the communication compression in order to reduce the energy consumption when computing and communicating the FL model.
Zaw et al.~\cite{zaw2021} set total energy consumption as a constraint when minimizing the training time of FL.

Finally, Anh et al.~\cite{anh2019} are the only ones to act on something similar to the workload distribution when optimizing the energy consumption of FL devices.
They use Deep Reinforcement Learning to decide, under uncertainty, how many data units and energy units a device should use for training in order to minimize the training time and energy consumption.
Still, instead of having a set workload to distribute, this scheme tries to balance using as much data in each device with using as little energy as possible.

\subsection{Energy Modeling and Profiling}

Many works in the FL community model the energy consumption or execution time of training with each data unit as constants~\cite{anh2019,kang2019,li2021,nguyen2021,tran2019,zaw2021,yang2021}.
This has the benefit of simplifying the problem at the cost of generality.
On the other hand, Khaleghzadeh, Manumachu, and Lastovetsky~\cite{khaleghzadeh2018} have shown that the actual performance of parallel applications running on heterogeneous resources may vary with workload size (i.e., costs are not constant).
Khaleghzadeh et al.~\cite{khaleghzadeh2021} have later shown that this is also the case for energy.
We employ a general model with an optimal solution in Sections~\ref{sec:problem} and~\ref{sec:general}, and we move to more constrained scenarios in Section~\ref{sec:opts}.

The energy consumed by a device during training is strongly dependent on its resources and the machine learning model being used.
For instance, Lane et al.~\cite{lane2015} have shown that the energy consumption for a single inference may vary between one and three orders of magnitude depending on the device and ML model.
Qiu et al.~\cite{qiu2021} illustrate differences in energy consumption during training dependending on the model, device, and FL strategy.
Additionally, besides hardware heterogeneity, the models being trained together may differ among devices, as is the case for personalized FL~\cite{jang2022}.

Obtaining accurate energy consumption information for scheduling on mobile devices can be done by monitoring the devices' utilization.
For instance, Walker et al.~\cite{walker2017} are able to model power consumption for multiple applications with low error~($<3.5\%$ on average) using performance monitoring counters available on mobile devices.
Kim and Wu~\cite{kim2021} use external power meters to gather power information from different devices at different clock frequencies and states (busy or idle).
They then organize devices at different power and performance levels, and use this information in combination with processor utilization information (from Unix commands) to model energy.

In the specific case of FL, I-Prof~\cite{damaskinos2020} has been proposed for profiling FL devices in order to predict the largest mini-batch size they can handle while respecting energy consumption and execution time limits. 
We believe I-Prof could be adapted to gather the energy consumption information required for our scheduling needs.
Another option lies in Flower~\cite{beutel2022}, a FL framework that already has the capacity to measure energy consumption on different devices, and that has been shown to be extensible~\cite{zhao2022}.

Finally, for a more general review of techniques to estimate the energy consumption of ML models, we point our readers to the survey by Garcia-Martin et al.~\cite{garcia2019}.

\section{Definition of the Scheduling Problem}\label{sec:problem}

We can combine the previous ideas related to workload distribution and energy optimization into our scheduling problem.
The problem of minimizing the energy consumption of heterogeneous devices during one round of Federated Learning can be defined in a similar fashion to the problem of minimizing the duration of a round~\cite{pilla2021}, which, by itself, resembles the problem of scheduling identical and independent tasks on heterogeneous resources~\cite[Chapter 6.1]{casanova2008parallel}.
From now on, we will refer to our mobile or Edge devices as \textbf{resources}, to the data units or mini-batches to schedule as \textbf{tasks}, and to the energy consumed by a device as its \textbf{cost}.
Table~\ref{tab:summary} summarizes the main notation used throughout this text. 

Consider a situation with $\NumRes$ heterogeneous resources organized in a set $\SetRes = \{1,2,\dots,\NumRes\}$.
Together, these resources must train their machine learning models with a workload of total size $\NumTasks \in \Natural$.
This workload is composed of identical, independent, and atomic tasks.

All resources have their own upper and lower limits on the number of tasks that they can use for training during the round.
We represent these upper and lower limits as $\SetUpper = \{\Upper{1},\dots, \Upper{\NumRes}\}$ and $\SetLower = \{\Lower{1},\dots,\Lower{\NumRes}\}$, respectively ($\Upper{i}$ and $\Lower{i} \in \Natural,~\forall i \in \SetRes$).
A resource $i \in \SetRes$ has its own local cost function $\Cost{i}:[\Lower{i},\Upper{i}]\rightarrow\Real$ that represents the energy that it consumes during training with a given number of tasks.
We use $\SetCosts = \{\Cost{1},\dots,\Cost{\NumRes}\}$ to represent the set of cost functions.
Finally, consider the schedule $\Mapping = \{\Map{1},\dots,\Map{\NumRes}\}$ that assigns $\Map{i} \in [\Lower{i},\Upper{i}]$ tasks to each resource~$i \in \SetRes$.

\begin{definition}[Minimal Cost FL Schedule]\label{def:problem}
    Given an instance $(\SetRes,\NumTasks,\SetUpper,\SetLower,\SetCosts)$, the goal is to find an optimal schedule $\Mapping^*$ that minimizes the total cost $\TotalCost$, i.e.:

    \begin{subequations}
        \begin{align}
            \text{minimize}_{\Mapping} & ~\TotalCost \coloneqq \sum\limits_{i \in \SetRes} \Cost{i}(\Map{i}) \label{eq:totalcost} \\
            \text{subject to} & ~\sum\limits_{i \in \SetRes} \Map{i} = \NumTasks, \label{eq:allmap} \\
            & ~\Lower{i} \le \Map{i} \le \Upper{i},~\forall i \in \SetRes \label{eq:all-respect}
        \end{align}
    \end{subequations}

\end{definition}

Throughout this text, we focus on non-trivial, valid problem instances.
In general, this means that no resource has an upper limit that is smaller than its lower limit, and that the number of tasks to assign is greater than the sum of lower limits and smaller than the sum of upper limits.
Without loss of generality, we also assume that $\NumTasks > \NumRes$.

\subsection{Problem Example}

Consider the situation where $\SetRes = \{1,2,3\}$, $\SetUpper = \{6,6,5\}$, $\SetLower = \{1,0,0\}$, and $\SetCosts = \{\{1:2, 2:3.5, 3:5.5, 4:8, 5:10, 6:12\},$ $\{0:0, 1:1.5, 2:2.5, 3:4, 4:7, 5:9, 6:11\}$, $\{0:0, 1:3, 2:4, 3:5, 4:6,$ $5:7\}\}$.
Figs.~\ref{fig:1} and~\ref{fig:2} illustrate this as Gantt charts, where each line represents a resource and the numbers in blue provide the local cost of assigning a given number of tasks to each resource.
The representation as Gantt charts is possible here because the costs are monotonically increasing, but this is not a constraint to our scheduling problem.

\begin{table}[t]
\centering
    \caption{Summary of main notation and symbols (by order of appearance).}
    \label{tab:summary}
    \begin{tabular}{cl}
\textbf{Symbol} & \multicolumn{1}{c}{\textbf{Meaning}} \\
\toprule
$\NumRes$ & Number of resources or disjoint classes of items. \\
$\SetRes$ & Set of resources. \\
$\NumTasks$ & Size of the workload or knapsack capacity. \\
$\Upper{i}$ & Upper limit of tasks of resource $i$. \\
$\Lower{i}$ & Lower limit of tasks of resource $i$. \\
$\Cost{i}(j)$ & Cost of assigning $j$ tasks to $i$. \\
$\Map{i}$ & Number of tasks assigned to resource $i$. \\
$\Mapping$ & Schedule assigning tasks to all resources. \\
$\Mapping^*$ & Optimal schedule. \\
$\TotalCost$ & Total cost of a schedule. \\
\midrule
$N_i$ & Disjoint class of items. \\
$c_{ij}$ & Cost of item $j$ from class $N_i$. \\
$w_{ij}$ & Weight of item $j$ from class $N_i$. \\
$x_{ij}$ & Binary variable for choosing item $j$ from class $N_i$. \\
$y$ & Occupied capacity of the knapsack. \\
$\mathcal{Z}_r$ & Partial solution value with the first $r$ item classes. \\
$\tau$ & Restricted knapsack capacity. \\
$\mathcal{X}(\NumTasks)$ & Optimal solution for \theMC{} with capacity $\NumTasks$. \\
$K$ & Minimal costs table for dynamic programming. \\
$I$ & Items table for dynamic programming. \\
$\NumTasks^*$ & Capacity used in the optimal solution of \theMC{}. \\
\midrule
$\Mar{i}(j)$ & Marginal cost of assigning the $j^{th}$ task to resource $i$. \\
$\Map{i}^t$ & Partial assignment of tasks to resource $i$.\\
$\Mapping^t$ & Schedule assigning $t$ tasks among the resources.\\
$\TotalCost^t$ & Total cost of a schedule with $t$ tasks. \\
$\SetRes^{lim}$ & Subset of resources with upper limits. \\
$\SetRes^{unl}$ & Subset of resources without upper limits. \\
$\NumRes^{lim}$ & Number of resources with upper limits. \\
$\gamma$ & Translation from disjoint classes to resources. \\
\end{tabular}
\end{table}

\begin{figure}[!htb]
    \centering 
    \includegraphics[width=0.5\columnwidth]{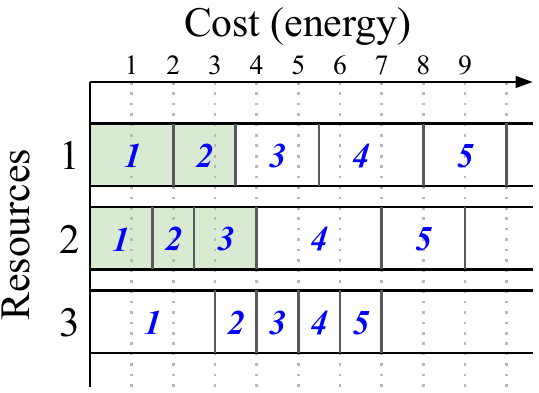}
    \caption{Gantt chart with the optimal schedule for $\NumTasks=5$ shaded in green.}
    \label{fig:1}
\end{figure}

\begin{figure}[!htb]
    \centering 
    \includegraphics[width=0.5\columnwidth]{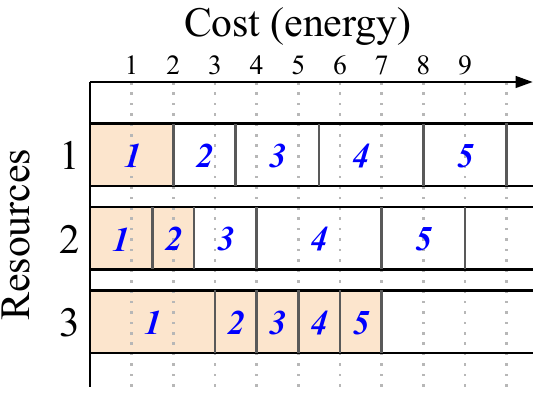}
    \caption{Gantt chart with the optimal schedule for $\NumTasks=8$ shaded in orange.}
    \label{fig:2}
\end{figure}

Fig.~\ref{fig:1} shows the optimal schedule when $\NumTasks=5$, i.e., $\Mapping^* = \{2,3,0\}$. 
It provides a total cost $\TotalCost=7.5$.
Although assigning all tasks to resource~$3$ would provide a smaller total cost, it would not respect $\Lower{1}$.
Meanwhile, Fig.~\ref{fig:2} illustrates the optimal schedule for $\NumTasks=8$, i.e., $\Mapping^* = \{1,2,5\}$ with $\TotalCost=11.5$.
This assignment reaches both $\Lower{1}$ and $\Upper{3}$.
We can also notice that the solution for the second problem does not contain the solution of the first smaller problem.
This provides us with the important insight that simple greedy algorithms will not find optimal schedules.

\section{Optimal Solution as a Knapsack Problem}\label{sec:general}


The Minimal Cost FL Schedule problem, as shown in Definition~\ref{def:problem}, is structured as a minimization problem where the solution has to include a valid item (schedule) from each class of items (possible schedules for each resource).
These properties can be directly mapped to a previously unexplored knapsack problem, which we will call the \textbf{Multiple-Choice Minimum-Cost Maximal Knapsack Packing Problem} (\textbf{\theMC}).
Throughout this section, we will explain this knapsack problem and show how it generalizes our scheduling problem.
We follow this discussion with the presentation of its recurrence function and an optimal dynamic programming~(DP) solution.

\subsection{The Multiple-Choice Minimum-Cost Maximal Knapsack Packing Problem}\label{subsec:mcmkp}


\theMC{} inherits characteristics from other knapsack problems.
It shows some similarities to the Multiple-Choice Knapsack Problem~(MCKP)~\cite{kellerer2004}, a maximization knapsack problem where the set of items is partitioned into classes.
Although MCKP in minimization form can be easily turned into an equivalent maximization problem, its optimal solution may not use the full capacity of the knapsack.
Meanwhile, the Minimum-Cost Maximal Knapsack Packing Problem~(MCMKP)~\cite{furini2017} is a less studied knapsack problem variation~\cite{cacchiani2022} where the objective is to fully occupy the knapsack while also trying to minimize its total cost.
Nonetheless, it lacks the classes present in MCKP.
In a sense, \theMC{} acts as a generalization of MCMKP, just as MCKP generalizes the ordinary 0-1 Knapsack Problem~\cite{kellerer2004}.

For the construction of \theMC, consider a set of $\NumRes$ disjoint classes $\mathcal{N} = \{N_1,\dots,N_{\NumRes}\}$ of items to be packed into a knapsack of capacity~$\NumTasks$.
Each item $j \in N_i$ has a cost $c_{ij} \in \Real$ and a weight $w_{ij} \in \Natural$.
The problem is to choose exactly one item from each class such that the cost sum is minimized while using the knapsack’s capacity to its maximum\footnote{Knapsack problem notations usually employ $C$ to represent the capacity of the knapsack and $p_{ij}$ to represent profits (instead of costs). We chose to use the notation $\NumTasks$ and $c_{ij}$ in order to preserve some similarity with our scheduling problem.}.
We use binary variables $x_{ij}$ that take on value~$1$ if and only if $j$ is chosen in class $N_i$, and an integer variable $y$ that represents the occupancy of the knapsack's capacity.
With this components in mind, we can define the \theMC{} as follows:

\begin{definition}[Multiple-Choice Minimum-Cost Maximum Knapsack Packing Problem]\label{def:mcmcmkp}
    Given a knapsack instance $(\mathcal{N}, c, w, \NumTasks)$, the goal of \theMC{} is to find a maximal knapsack packing that minimizes the cost of selected items, i.e.:

\begin{subequations}
    \begin{align}
        \text{minimize}_{x,y}&~~~\sum\limits_{i=1}^{\NumRes}\sum\limits_{j \in N_i} c_{ij}x_{ij} -y\sum\limits_{i=1}^{\NumRes}\sum\limits_{j \in N_i} c_{ij},\label{ks:min}\\ 
        \text{subject to}&~~~\sum\limits_{i=1}^{\NumRes}\sum\limits_{j \in N_i} w_{ij}x_{ij} \leq \NumTasks, \label{ks:capacity}\\
        &~~~\sum\limits_{i=1}^{\NumRes}\sum\limits_{j \in N_i} w_{ij}x_{ij} \geq y,\label{ks:occupancy}\\
        &~~~\sum\limits_{j \in N_i} x_{ij} = 1,~~i=1,\dots,\NumRes, \label{ks:one_item} \\
        &~~~x_{ij} \in \{0,1\},~~i=1,\dots,\NumRes,~~j\in N_i, \\
        &~~~y \in [0,\NumTasks]
    \end{align}
\end{subequations}
\end{definition}

The formulation in Definition~\ref{def:mcmcmkp} is similar to the formulation of MCKP~\cite{kellerer2004}.
For instance, rule~\eqref{ks:capacity} constrains the solutions to the ones that fit into the knapsack, and rule~\eqref{ks:one_item} says that a single item from each set must be included.
The outliers here are rules~\eqref{ks:min} and~\eqref{ks:occupancy} that adapt the idea of a maximal knapsack packing~\cite{furini2017}.
Given that maximizing the occupancy of the knapsack has precedence over the minimal cost, rule~\eqref{ks:min} unifies both objectives by multiplying $y$ by a negative constant of absolute value larger than any possible minimal cost.
Despite this, the actual cost of the solution relates only to the $\sum\sum c_{ij}x_{ij}$ part of rule~\eqref{ks:min}.


\subsubsection{Transformation from Scheduling to Knapsack Problem}

Before going through the solution to this problem, we would like to emphasize how \theMC{} generalizes the Minimal Cost FL Schedule problem presented in Section~\ref{sec:problem}.

As a first step, we can focus on the transformation between problem instances $(\mathcal{N}, c, w, \NumTasks)$ and $(\SetRes, \NumTasks,\SetUpper,\SetLower,\SetCosts)$.
This is based on the idea that a disjoint class $N_i$ can be composed by all possible schedules for resource $i \in \SetRes$, i.e., $N_i = \{\Lower{i}, \Lower{i}+1, \dots, \Upper{i}\}$.
In this situation, the cost and weight of an item $j\in N_i$ are set as $c_{ij} = \Cost{i}(j)$ and $w_{ij} = j$.
The solution of \theMC{} can be transformed to its scheduling equivalent by setting $x_i = j,~\forall i \in \SetRes, x_{ij} = 1$.

Concerning the generalization, \theMC{} relaxes two different constraints from the Minimal Cost FL Schedule problem.
The first difference is that the classes in \theMC{} can contain items with any arbitrary weights, while our scheduling problem considers all solutions in a given interval with upper and lower limits.
The second difference relates to~$\NumTasks$: \theMC{} accepts packings that use less than the total capacity of the knapsack, while our scheduling problem requires assigning all tasks to resources.
These two distinctions are related: as our scheduling problem considers solution intervals for each resource, there are always solutions that assign all tasks.
Meanwhile, given the arbitrary weights of items in \theMC{}, there may be no solution that fully occupies the knapsack, leading to the second relaxation.
Nonetheless, we can count on an optimal solution to \theMC{} to also optimize the Minimal Cost FL Schedule problem.

\subsection{Dynamic Programming Solution}\label{subsec:dp}

Optimal solutions to \theMC{} can be found using a dynamic programming technique similar to the one used for MCKP~\cite{kellerer2004,dudzinski1987}.
Let $\mathcal{Z}_r(\tau)$ be an optimal solution value for a partial problem with the first $r$ item classes that fully occupies a knapsack with restricted capacity $\tau$.
Assume that $\mathcal{Z}_r(\tau)\coloneqq\infty$ if no solution exists and that $\mathcal{Z}_0(0)\coloneqq 0$.
$\mathcal{Z}_r(\tau)$ is defined in~\eqref{eq:mcopt} and its value can be recursively computed following~\eqref{eq:recursion}.

\begin{equation}
\label{eq:mcopt}
\mathcal{Z}_r(\tau) \coloneqq \min \left\{ \sum\limits_{i=1}^{r}\sum\limits_{j \in N_i} c_{ij}x_{ij}\middle\vert\
    \begin{array}{l}
    \sum\limits_{i=1}^{r}\sum\limits_{j \in N_i} w_{ij}x_{ij} = \tau,\\
    \sum\limits_{j \in N_i} x_{ij} = 1,i=1,\dots,r, \\
    x_{ij} \in \{0,1\},i=1,\dots,r,\\
    ~~~~~~~~~~~~~~~~~~~~j\in N_i 
    \end{array}
\right\}
\end{equation}

\begin{equation}
\label{eq:recursion}
\mathcal{Z}_r(\tau) = \min\limits_{j \in N_r, w_{rj}\leq \tau} 
    \left( \mathcal{Z}_{r-1}(\tau-w_{rj})+c_{rj} \right)
\end{equation}

Using the optimal solutions for partial problems found by $\mathcal{Z}$, we define $\mathcal{X}(\NumTasks)$ as the optimal solution for \theMC{} with capacity $\NumTasks$ as shown in~\eqref{eq:finalrec}.
Simply put, it takes the optimal solution for a capacity $\NumTasks$ if it exists, or it takes the optimal solution found with the highest occupancy. 
The combined use of $\mathcal{X}$ and $\mathcal{Z}$ lets us find a maximum knapsack packing that also provides the minimum cost.

\begin{equation}
    \label{eq:finalrec}
    \mathcal{X}(\NumTasks) =
    \begin{cases}
        \mathcal{Z}_\NumRes(\NumTasks) & \text{if}~ \mathcal{Z}_\NumRes(\NumTasks)\neq\infty,\\
        \mathcal{X}(\NumTasks-1) & \text{otherwise}
    \end{cases}
\end{equation}

We use the ideas behind~\eqref{eq:recursion} and~\eqref{eq:finalrec} to propose Algorithm~\ref{algo:1}, which presents a dynamic programming implementation for the optimal solution of \theMC.
Algorithm~\ref{algo:1} uses two matrices, $K$ and $I$, to store the minimal costs that are progressively computed and the items that are part of these solutions, respectively.
The algorithm first stores all possible solutions for $\mathcal{Z}_1$ (lines 7--9) and then iteratively computes the optimal solutions for increasing numbers of item classes and all knapsack capacities (lines 10--19).
By its end, it finds the highest knapsack occupancy possible and its minimum cost (lines 21--24), and it goes in the reverse order through the item classes to extract the items that belong in the optimal solution (lines 25--28).

Algorithm~\ref{algo:1} shows a space bound in $O(\NumTasks\NumRes)$ and it requires $O(\NumTasks\sum_{i=1}^\NumRes |N_i|)$ operations.
Both are equivalent to the complexity found in the DP solution for MCKP~\cite{kellerer2004}.
In the context of our scheduling problem, we can have at most $\NumTasks$ assignments possible for each resource.
This gives us a worst-case complexity in $O(\NumTasks^2\NumRes)$.

\begin{algorithm}[h]
    \caption{A DP solution to \theMC.}
    \label{algo:1}
    \begin{algorithmic}[1]
\Require Set of disjoint classes of items $\mathcal{N} = \{N_1,\dots,N_{\NumRes}\}$ with costs $c_{ij}$ and weights $w_{ij}$, $i=1,\dots,\NumRes$, $j\in N_i$. Knapsack capacity $\NumTasks$.
\Ensure Total cost $\TotalCost$, required capacity $\NumTasks^*$, and list of items in the solution $\Mapping = \{\Map{1},\dots,\Map{\NumRes}\}$.
        \For{$i=1,\dots,\NumRes$}
        \Comment{\textit{Initialization of minimal costs}}
            \For{$t=0,\dots,\NumTasks$}
            \Comment{\textit{and partial solutions matrices.}}
            \State $K[i][t] \leftarrow \infty$~;~$I[i][t] \leftarrow \varnothing$ 
            \EndFor
        \State $\Map{i} \leftarrow \varnothing$ 
        \EndFor
        \For{$j \in N_1$}
            \Comment{\textit{Only solutions for} $\mathcal{Z}_1$\textit{.}}
            \State $K[1][w_{1j}] \leftarrow c_{1j}$~;~$I[1][w_{ij}] \leftarrow j$ 
        \EndFor
        \For{$i=2,\dots,\NumRes$}
        \Comment{\textit{Computes} $\mathcal{Z}_i$ \textit{for all capacities.}}
            \For{$j\in N_i$}
            \Comment{\textit{Using all items in}~$N_i$.}
                \For{$t=w_{ij},\dots,\NumTasks$}
                \If{$K[i-1][t-w_{ij}] + c_{ij} < K[i][t]$}
                    \Statex \Comment{\textit{Best solution for} $\mathcal{Z}_i(t)$ \textit{so far.}}
                    \State $K[i][t] \leftarrow K[i-1][t-w_{ij}] + c_{ij}$
                    \State $I[i][t] \leftarrow j$
                \EndIf
                \EndFor
            \EndFor
        \EndFor
        \State $\NumTasks^* \leftarrow \NumTasks$
        \While{$K[\NumRes][\NumTasks^*] = \infty$}
        \Comment{\textit{Finds} $\NumTasks^*$.}
            \State $\NumTasks^* \leftarrow \NumTasks^* -1$
        \EndWhile
        \State $\TotalCost \leftarrow K[\NumRes][\NumTasks^*]$ 
        \Comment{\textit{Finds} $\TotalCost$.}
        \State $t \leftarrow \NumTasks^*$
        \For{$i=\NumRes,\dots,1$}
            \Comment{\textit{Finds} $\Mapping$.}
            \State $j \leftarrow I[i][t]$~;~ $\Map{i} \leftarrow j$~;~ $t \leftarrow t - w_{ij}$
        \EndFor
        \State \textbf{return} $\TotalCost$, $\NumTasks^*$, $\Mapping$
    \end{algorithmic}
\end{algorithm}

\subsubsection{Proof of optimality}

The optimality of Algorithm~\ref{algo:1} can be easily demonstrated by induction and is kept here for the sake of completeness.
We first highlight the optimality of the base case for $\mathcal{Z}_1$ in Lemma~\ref{proof:1:1}.
We then prove the induction step in Lemma~\ref{proof:1:2}.
Finally, we combine these ideas with the selection of the minimum cost solution with the highest knapsack occupancy in Theorem~\ref{proof:1:3}.

\begin{lemma}\label{proof:1:1}
    The solutions in $\mathcal{Z}_1$ are optimal.
\end{lemma}
\begin{proof}
    The only possible solutions for $\mathcal{Z}_1$ are $\mathcal{Z}_1(w_{1j})=c_{1j}$ for all $j \in N_1$, therefore they are optimal.
\end{proof}

\begin{lemma}\label{proof:1:2}
    If the solutions in $\mathcal{Z}_i$ are optimal, then the solutions in $\mathcal{Z}_{i+1}$ are also optimal.
\end{lemma}
\begin{proof}
    By the definition in~\eqref{eq:recursion}, we know that the value of $\mathcal{Z}_{i+1}(\tau)$ for $\tau \in [0,\NumTasks]$ will be the minimum among all valid combinations of solutions from $\mathcal{Z}_i$ (which are all minimal) and items from $N_{i+1}$ that, together, fill a knapsack with capacity $\tau$.
    To consider that there would be another solution with a smaller value would be absurd, as $\mathcal{Z}_{i+1}(\tau)$ already takes the minimum value among all valid solutions. Therefore, the solutions for $\mathcal{Z}_{i+1}$ are optimal.
\end{proof}

\begin{theorem}\label{proof:1:3}
    $\mathcal{X}(\NumTasks)$ provides the optimal solution for \theMC.
\end{theorem}
\begin{proof}
    Lemmas~\ref{proof:1:1} and~\ref{proof:1:2} prove the optimality of the base case and the induction step for $\mathcal{Z}$, therefore the solutions provided by $\mathcal{Z}$ are optimal (i.e., minimal).
    $\mathcal{X}(\NumTasks)$ returns the solution of $\mathcal{Z}(\tau)$ for the highest value of $\tau \in [0,\NumTasks]$, thus it provides the minimum cost for the maximal knapsack packing possible, making its solution for \theMC{} optimal.
\end{proof}

\section{Optimizations for Scenarios with Monotonically Increasing Cost Functions}\label{sec:opts}


As we have seen previously, an optimal solution for the Minimal Cost FL Schedule can be found in pseudo-polynomial time in $O(\NumTasks^2\NumRes)$ for any arbitrary, valid problem instance.
Algorithm~\ref{algo:1} makes no special assumptions regarding the behavior of cost functions in relation to the number of tasks assigned to a given resource.
Meanwhile, other works in the state of the art model the execution time or energy consumption of FL tasks as linearly-proportional to the number of tasks~\cite{anh2019,kang2019,li2021,nguyen2021,tran2019,zaw2021,yang2021}.
This kind of assumption has a direct impact on an algorithm's design and it can lead to simpler, faster solutions.

In this section, we discuss more optimized solutions for variants of our scheduling problem where the cost functions of all resources increase monotonically and follow the same behavior.
Section~\ref{subsec:extradefs} provides additional definitions related to marginal costs that affect the choice of algorithm.
Section~\ref{subsec:low} includes a simple set of rules to remove the lower limits on our scheduling problem in order to simplify the presentation of our algorithms.
Sections~\ref{subsec:optinc} and~\ref{subsec:optcons} present optimal algorithms for increasing and constant marginal costs, respectively.
Finally, Sections~\ref{subsec:optdecunl} and~\ref{subsec:optdec} focus on the optimal solutions for problems with decreasing marginal costs without and with upper limits.


\subsection{Additional Definitions and Notation}\label{subsec:extradefs}

Given a resource $i \in \SetRes$ and its respective cost function $\Cost{i}$, we can define its marginal cost function $\Mar{i}:[\Lower{i},\Upper{i}]\rightarrow\Real$ to represent the cost of assigning each additional task to~$i$ based on~\eqref{eq:mar}.
By considering only non-negative marginal costs, we are also enforcing situations where all cost functions are monotonically increasing.

\begin{equation}
\label{eq:mar}
\Mar{i}(j) = 
\begin{cases}
0 & \text{if}~ j = \Lower{i},\\
\Cost{i}(j)-\Cost{i}(j-1) & \text{otherwise}
\end{cases}
\end{equation}

Marginal costs are useful to synthesize and differentiate cost behaviors into three classes of interest: increasing (convex), constant, and decreasing (concave).
These problems are characterized in Definition~\ref{def:mar}.
In other words, they represent situations where --- due to the specific hardware, software, and machine learning methods employed --- the energy consumption of each resource grows in a superlinear, linear, or sublinear fashion with the increase in the number of tasks used for training.

\begin{definition}[Increasing, Constant, and Decreasing Marginal Costs Problems]\label{def:mar}
Given a problem instance $(\SetRes,\SetUpper,\SetLower,\SetCosts)$, it is said to have increasing, constant, or decreasing marginal costs if and only if it follows~\eqref{eq:marin},~\eqref{eq:marco}, or~\eqref{eq:mardec}, respectively.

\begin{subequations}
\begin{align}
\text{increasing:} &~ \Mar{i}(j) \leq \Mar{i}(j+1), &\forall i \in \SetRes, j \in ~]\Lower{i},\Upper{i}[\label{eq:marin}\\
\text{constant:} &~\Mar{i}(j) = \Mar{i}(j+1), &\forall i \in \SetRes, j \in ~]\Lower{i},\Upper{i}[\label{eq:marco}\\
\text{decreasing:} &~\Mar{i}(j) \geq \Mar{i}(j+1), &\forall i \in \SetRes, j \in ~]\Lower{i},\Upper{i}[\label{eq:mardec}
\end{align}
\end{subequations}
\end{definition}


Finally, we represent the partial schedule of $t<\NumTasks$ tasks as $\Mapping^t = \{\Map{1}^t,\dots,\Map{\NumRes}^t\}$.
In this situation, $\Map{i}^t$ represents the number of tasks specifically assigned to resource~$i$, and $\TotalCost^t$ represents the total cost of this schedule.


\subsection{Simplification by Lower Limit Removal}\label{subsec:low}

Independently of the specific scheduling scenario being treated, any valid solution is required to give each resource a number of tasks that respects its lower limit, as given by~\eqref{eq:allmap}.
In this sense, we can transform any problem instance $(\SetRes,\NumTasks,\SetUpper,\SetLower,\SetCosts)$ into an equivalent instance $(\SetRes,\NumTasks',\{\Upper{1}',\dots,\Upper{\NumRes}'\},\{0\}^\NumRes,\{\Cost{1}',\dots,\Cost{\NumRes}'\})$ that simplifies the problem by shifting all lower limits to zero and adapting all other related values.
These operations are resumed in~\eqref{eq:numtasks0},~\eqref{eq:upper0}, and~\eqref{eq:cost0}.
The resulting schedule for the equivalent instance, $\Mapping' = \{\Map{1}',\dots,\Map{\NumRes}'\}$ can be transformed back to the original instance following~\eqref{eq:sched0}.
This simple set of rules represents a number of operations in $O(\NumRes)$ for any optimal solution, as the cost functions in~\eqref{eq:cost0} can be computed only when necessary. 

\begin{equation}\label{eq:numtasks0}
\NumTasks' \coloneqq \NumTasks - \sum\limits_{i \in \SetRes} \Lower{i}
\end{equation}

\begin{equation}\label{eq:upper0}
\Upper{i}' \coloneqq \Upper{i} - \Lower{i},~\forall i \in \SetRes
\end{equation}

\begin{equation}\label{eq:cost0}
\Cost{i}'(j) \coloneqq \Cost{i}(j+\Lower{i}) - \Cost{i}(\Lower{i}),~\forall i \in \SetRes,~j \in [0,\Upper{i}']
\end{equation}

\begin{equation}\label{eq:sched0}
\Map{i} \coloneqq \Map{i}' + \Lower{i}
\end{equation}

This simplification ensures that all scheduling algorithms start with an initial schedule $\Mapping^{l}$ for $l = \sum_{i \in \SetRes} \Lower{i}$.
This schedule is trivially shown to be optimal in Lemma~\ref{proof:baseline} and used for other optimality proofs later in this section.

\begin{lemma}\label{proof:baseline}
The partial schedule $\Mapping^{l}$ is optimal.
\end{lemma}
\begin{proof}
    $\Mapping^{l}$ is the only valid schedule that respects the lower limits of all resources, therefore it is optimal.
\end{proof}

\subsection{Increasing Marginal Costs (MarIn)}\label{subsec:optinc}

Our first scenario of interest considers the situation where the marginal costs in all resources are monotonically increasing.
In this scenario, we can minimize the total cost $\TotalCost$ by adapting a solution previously employed to minimize the maximum resource cost (in its context, the makespan).
The solution, called MarIn and adapted from OLAR~\cite{pilla2021}, is described in Algorithm~\ref{algo:marin}.
Its main idea is to assign the next task $t+1$ to a resource $i$ that has the minimal marginal cost $\Mar{i}(\Map{i}^t +1)$ (instead of the minimal cost, as originally done by OLAR) and that has not yet reached its upper limit (lines~5--6).

\begin{algorithm}[h]
    \caption{MarIn --- adapted from OLAR~\cite{pilla2021}.}
    \label{algo:marin}
    \begin{algorithmic}[1]
        \Require Set of resources $\SetRes$, number of tasks to schedule $\NumTasks$, set of upper limits $\SetUpper$, set of cost functions $\SetCosts$.
        \Ensure Optimal schedule $\Mapping$.
        \ForAll{$i \in \SetRes$}
            \State $\Map{i} \leftarrow 0$
            \Comment{\textit{All resources start without any tasks.}}
        \EndFor
        \For{$t=1,\dots,\NumTasks$}
            \State $k \leftarrow \argmin_{i \in \SetRes ,~ \Map{i} < \Upper{i}} \Mar{i}(\Map{i}+1)$
            \State $\Map{k} \leftarrow \Map{k}+1$
        \EndFor
        \State \textbf{return} $\Mapping$
    \end{algorithmic}
\end{algorithm}

Algorithm~\ref{algo:marin} has a space bound in $O(\NumRes)$ and it is computed using $\Theta(\NumRes + \NumTasks \log \NumRes)$ operations.
This complexity is achieved by employing a minimum binomial heap for storing the next task assignments (line~5).
This heap’s insertion and removal of the minimal item operations are in $\Theta(1)$ and $\Theta(\log \NumRes)$, respectively.

\subsubsection{Proof of Optimality}

MarIn's optimality can be proved by induction by combining Lemmas~\ref{proof:baseline} and~\ref{proof:stepmarin} in Theorem~\ref{proof:marin}.

\begin{lemma}\label{proof:stepmarin}
    If $\Mapping^t$ is optimal, then $\Mapping^{t+1}$ computed by MarIn is optimal.
\end{lemma}
\begin{proof}
    By definition, MarIn assigns task $t+1$ to a resource with minimum marginal cost $\Mar{i}(\Map{i}^t +1)$ for $\Map{i}^t \in \Mapping^t$ and $\Map{i}^t < \Upper{i}$.
    As all marginal cost functions are monotonically increasing~\eqref{eq:marin}, all previous assignments in $\Mapping^t$ had equal or smaller marginal costs.
    This means that $\Mapping^{t+1}$ schedules $t+1$ tasks to the resources with the smallest marginal costs.
    This makes its $\TotalCost^{t+1}$ minimal and, therefore, optimal.
\end{proof}

\begin{theorem}\label{proof:marin}
    The schedule $\Mapping$ computed by MarIn is optimal.
\end{theorem}
\begin{proof}
    Lemmas~\ref{proof:baseline} and~\ref{proof:stepmarin} prove the optimality of the base case and the induction step, thus the solution provided by MarIn is optimal. 
\end{proof}

\subsection{Constant Marginal Costs (MarCo)}\label{subsec:optcons}

As constant marginal costs are also monotonically increasing, this scenario could be optimally solved by MarIn.
Nonetheless, having constant costs facilitates the scheduling decisions, as we can decide to assign more than one task at each step.
This optimization is present in MarCo and illustrated in Algorithm~\ref{algo:marco}.

\begin{algorithm}[h]
    \caption{MarCo.}
    \label{algo:marco}
    \begin{algorithmic}[1]
        \Require Set of resources $\SetRes$, number of tasks to schedule $\NumTasks$, set of upper limits $\SetUpper$, set of cost functions $\SetCosts$.
        \Ensure Optimal schedule $\Mapping$.
        \ForAll{$i \in \SetRes$}
            \State $\Map{i} \leftarrow 0$
            \Comment{\textit{All resources start without any tasks.}}
        \EndFor
        \State $t\leftarrow 0$
        \While{$t<\NumTasks$}
            \State $k \leftarrow \argmin_{i \in \SetRes,~ \Map{i} \neq \Upper{i}} \Mar{i}(1)$
            \State $\Map{k} \leftarrow \min(\Upper{k}, \NumTasks-t)$
            \Comment{\textit{Assigns the most tasks possible.}}
            \State $t \leftarrow t + \Map{k}$
        \EndWhile
        \State \textbf{return} $\Mapping$
    \end{algorithmic}
\end{algorithm}

Algorithm~\ref{algo:marco}'s main loop (lines~5--9) uses the knowledge of constant marginal costs to find, on each of its iterations, an available resource with minimal marginal costs (line~6).
This resource is assigned the maximum number of remaining tasks possible (line~7) (i.e., its upper limit or all the remaining tasks), making this resource unavailable further on.
An iteration finishes with an update to the number of assigned tasks (line~8), and the loop finishes when no tasks remain to be scheduled.

MarCo displays the same space bound of MarIn but it only requires $\Theta(\NumRes \log \NumRes)$ operations.
This complexity is achieved by organizing the marginal costs of all resources (used on line~6) in a sorted list, so any searches in the main loop require only a constant number of operations.

\subsubsection{Proof of optimality}

MarCo's optimality can be proved similarly to MarIn before it.
Lemma~\ref{proof:stepmarco} proves that each step of the algorithm is optimal, and Theorem~\ref{proof:marco} uses this information to convey that the algorithm is optimal.

\begin{lemma}\label{proof:stepmarco}
    If $\Mapping^t$ is optimal, then $\Mapping^{t+a}$ computed by MarCo is optimal.
\end{lemma}
\begin{proof}
    By definition, MarCo assigns the next $a$ tasks to a resource with minimum marginal cost $\Mar{i}(1)$ for $i \in \SetRes$ and $\Map{i}^t \neq \Upper{i}$.
    By~\eqref{eq:marco}, the marginal costs are constant, so any new assignments to other available resources would have equal or greater marginal costs.
    Additionally, all previous assignments in $\Mapping^t$ had smaller or equal marginal costs.
    This means that $\Mapping^{t+a}$ schedules $t+a$ tasks to the resources with the smallest marginal costs.
    This makes its $\TotalCost^{t+a}$ minimal and, therefore, optimal.
\end{proof}

\begin{theorem}\label{proof:marco}
    The schedule $\Mapping$ computed by MarCo is optimal.
\end{theorem}
\begin{proof}
    Lemmas~\ref{proof:baseline} and~\ref{proof:stepmarco} prove the optimality of the base case and the induction step, thus the solution provided by MarCo is optimal. 
\end{proof}

\subsection{Decreasing Marginal Costs without Upper Limits (MarDecUn)}\label{subsec:optdecunl}

The presence of decreasing marginal costs requires an approach that is different from previous scenarios.
While previous scenarios made it possible to incrementally assign the tasks with the smallest marginal costs, here the smallest marginal costs come from the last tasks assigned to a resource.
In this sense, optimal assignments can include tasks with high marginal costs if enough tasks are assigned to the same resource, reducing the average cost per task.

Algorithm~\ref{algo:mardecun}, named MarDecUn, focuses on the situation where the upper limits of all resources are equal or superior to the actual number of tasks to schedule.
In this situation, the optimal solution is found by simply scheduling all tasks to a resource with minimal average cost per task or, in other words, a resource with minimum cost for $\NumTasks$ tasks (lines 4--5).

\begin{algorithm}[h]
    \caption{MarDecUn.}
    \label{algo:mardecun}
    \begin{algorithmic}[1]
        \Require Set of resources $\SetRes$, number of tasks to schedule $\NumTasks$, set of upper limits $\SetUpper$, set of cost functions $\SetCosts$.
        \Ensure Optimal schedule $\Mapping$.
        \ForAll{$i \in \SetRes$}
            \State $\Map{i} \leftarrow 0$
            \Comment{\textit{All resources start without any tasks.}}
        \EndFor
        \State $k \leftarrow \argmin_{i \in \SetRes} \Cost{i}(\NumTasks)$
        \State $\Map{k} \leftarrow \NumTasks$
        \Comment{\textit{Assigns all tasks to the same resource.}}
        \State \textbf{return} $\Mapping$
    \end{algorithmic}
\end{algorithm}

Due to its simplicity, MarDecUn requires only $\Theta(\NumRes)$ operations to find a resource with minimal cost.
Its space bound is still the same $O(\NumRes)$ of previous algorithms because its solution includes a schedule for all resources. If its output were to be changed to inform only the resource receiving all tasks, this bound could be reduced to a constant.

\subsubsection{Proof of optimality}

MarDecUn's optimality can be proved based on two ideas.
The first idea is that MarDecUn assigns all possible tasks to a resource with minimal cost.
The second idea is related to the behavior of decreasing functions, which is presented in Lemma~\ref{lemma:non}.
These ideas are combined in Theorem~\ref{proof:mardecun}.

\begin{lemma}[Sum of contiguous intervals of decreasing functions]\label{lemma:non}
    If $f,g:~\Natural \to \mathbb{R}$ are monotonically decreasing functions and $f(s_f+1) \leq g(s_g+1)$, then~\eqref{eq:f} is true for any intervals $[i_f,s_f+s_g-i_g+1]$ and $[i_g,s_g+s_f-i_f+1]$.

\begin{equation}
    \sum\limits_{i=i_f}^{s_f} f(i)+\sum\limits_{i=i_g}^{s_g} g(i)
    \ge \sum\limits_{i=i_f}^{s_f+s_g-i_g+1} f(i)
    \label{eq:f}
\end{equation}

\end{lemma}
\begin{proof}
    By definition, the decreasing functions $f$ and $g$ follow the behavior illustrated in~\eqref{eq:noni}. 
    \begin{equation}
        \begin{split}
            f(i_f) \ge \dots \ge f(s_f+1)\ge \dots \ge f(s_f+s_g-i_g+1) \\
            g(i_g) \ge \dots \ge g(s_g+1)\ge \dots \ge g(s_g+s_f-i_f+1) 
        \end{split}
        \label{eq:noni}
    \end{equation}
    
    If $g(s_g+1) \ge f(s_f+1)$, then this should also hold for all values larger than $s_f+1$ and smaller than $s_g+1$~\eqref{eq:gf}.
    \begin{equation}
    g(i_g) \ge \dots \ge g(s_g+1)\ge f(s_f+1)\ge \dots \ge f(s_f+s_g-i_g+1)
        \label{eq:gf}
    \end{equation}

    As $g(s_g)\ge f(s_f+1)$, changing the intervals to $[i_f,s_f+1]$ and $[i_g,s_g-1]$ should lead to a smaller or equal sum without changing the number of elements considered. This idea can be applied iteratively, leading to~\eqref{eq:allsums} and proving~\eqref{eq:f}.
    \begin{equation}
        \begin{split}
            \sum\limits_{i=i_f}^{s_f} f(i)+\sum\limits_{i=i_g}^{s_g} g(i) \ge&~ \sum\limits_{i=i_f}^{s_f+1} f(i)+\sum\limits_{i=i_g}^{s_g-1} g(i)\\
            \ge \dots \ge& \sum\limits_{i=i_f}^{s_f+s_g-i_g} f(i)+\sum\limits_{i=i_g}^{i_g} g(i)\\
             \ge &\sum\limits_{i=i_f}^{s_f+s_g-i_g+1} f(i)
        \end{split}
    \label{eq:allsums}
    \end{equation}
\end{proof}

\begin{theorem}\label{proof:mardecun}
    The schedule $\Mapping$ computed by MarDecUn is optimal.
\end{theorem}
\begin{proof}
    MarDecUn assigns all $\NumTasks$ tasks to a resource with minimal cost, so no other assignment to a single resource could improve the solution.
    By the definition of marginal costs in~\eqref{eq:mar}, we can rewrite the cost of a schedule for one resource as a sum of marginal costs, as shown in~\eqref{eq:mar2}.

\begin{equation}
\label{eq:mar2}
\Cost{i}(\NumTasks) = \Cost{i}(0) + \sum\limits_{t=1}^{\NumTasks} \Mar{i}(t),~~\forall i \in \SetRes
\end{equation}

    Given that all marginal cost functions are decreasing, Lemma~\ref{lemma:non} tells us that no solution splitting the $\NumTasks$ tasks among multiple resources can provide a smaller total cost, therefore the schedule computed by MarDecUn is minimal and optimal.
\end{proof}

\subsection{Decreasing Marginal Costs with Upper Limits (MarDec)}\label{subsec:optdec}

Although the previous solution for the scenario of decreasing marginal costs and no upper limits cannot be applied in the presence of upper limits, it provides us with an important insight.
Namely, Lemma~\ref{lemma:non} tells us that it is always more beneficial to put all tasks in the same resource, so an optimal solution can be found in one of two scenarios: 
(I) all tasks are assigned to a resource without upper limits; or 
(II) all tasks are assigned only to resources at maximum capacity and at most one resource at intermediary capacity.
Both scenarios are covered by MarDec in Algorithm~\ref{algo:mardec}.

MarDec starts by splitting the resources into two subsets: one for the resources that have upper limits ($\SetRes^{lim}$) and one for those who do not ($\SetRes^{unl}$) (lines 1--2).

In order to compute possible solutions for scenario~(II), MarDec employs a dynamic programming solution for the Minimum-Cost Maximal Knapsack Packing~(MCMKP) problem~\cite{furini2017}, as it helps us find which resources have to be at maximum capacity.
As MCMKP is a specialization of \theMC, Algorithm~\ref{algo:mardec} uses a variation of Algorithm~\ref{algo:1} that we call \theMC-matrices.
It outputs the support matrices $K$ and $I$ to enable the reuse of its partial solutions.
As previously described in Section~\ref{subsec:dp}, $K$ stores the minimal costs that were progressively computed, while $I$ stores the items that are part of the partial solutions.
MarDec also employs Algorithm~\ref{algo:prep} to convert its variables for use in a knapsack problem, and Algorithm~\ref{algo:trans} to translate a partial MCMKP solution to a schedule.

MarDec covers scenario~(II) in two steps.
At first, it computes all possible minimal solutions where a resource without upper limits is set at intermediary capacity (lines~6--15).
For each possible intermediary capacity (line~8), it finds the resource with minimal cost to receive the remaining tasks.
It is important to emphasize here that, if no solution is found for a specific knapsack capacity, \theMC{} provides an infinite cost, so no invalid solutions are ever considered.
Additionally, when $t=\NumTasks$, the solution for scenario~(I) is computed (i.e., the whole workload goes to a single resource as in MarDecUn).
At its second step, MarDec verifies all possible minimal solutions where one of the resources with upper limits ends up at intermediary capacity (lines 17--28).
In this case, MarDec removes the resource of interest from the input of \theMC~(line 18) and then computes all possible optimal schedules.
Throughout all these steps, the schedule $\Mapping$ that provides the minimal cost $\TotalCost$ is kept and provided at the end of the algorithm.


\begin{algorithm}[h]
    \caption{MarDec.}
    \label{algo:mardec}
    \begin{algorithmic}[1]
        \Require Set of resources $\SetRes$, number of tasks to schedule $\NumTasks$, set of upper limits $\SetUpper$, set of cost functions $\SetCosts$.
        \Ensure Optimal schedule $\Mapping$.
        \State $\SetRes^{lim} \leftarrow \{i\},~\forall i \in \SetRes, \Upper{i} < \NumTasks$
        \Comment{\textit{Resources w/ upper limits}}
        \State $\SetRes^{unl} \leftarrow \SetRes \setminus \SetRes^{lim}$
        \Comment{\textit{Resources without upper limits}}
        \State $\NumRes^{lim} \leftarrow |\SetRes^{lim}|$
        \State $\TotalCost \leftarrow \infty$
        \Comment{\textit{No valid solution so far.}}
        \Statex $\triangleright$~\textit{Resource from } $\SetRes^{unl}$ \textit{at intermediary capacity.}
        \If{$\SetRes^{unl} \neq \emptyset$}        

        \State $(\mathcal{N}, c, w, \gamma) \leftarrow$ Prepare$(\SetRes^{lim}, \SetUpper, \SetCosts)$
        \Comment{\textit{Algorithm~\ref{algo:prep}}.}
        \State $(K,I)\leftarrow$~\theMC-matrices$(\mathcal{N},c,w,\NumTasks)$
        \For{$t=0,\dots,\NumTasks$}
        \Comment{\textit{Evaluates all partial solutions.}}
            \State $k \leftarrow \argmin_{i \in \SetRes^{unl}} \Cost{i}(t)$
            \If{$\Cost{k}(t) + K[\NumRes^{lim}][\NumTasks-t] < \TotalCost$}
                \State $\TotalCost \leftarrow\Cost{k}(t) + K[\NumRes^{lim}][\NumTasks-t]$
                \Comment{\textit{New minimal.}}
                \State $\Mapping \leftarrow$ Translate$(\gamma,\SetRes,\mathcal{N},w,I,t)$
                \Comment{\textit{Algorithm~\ref{algo:trans}.}}
                \State $\Map{k} \leftarrow t$
            \EndIf
        \EndFor
        \EndIf
        \Statex $\triangleright$~\textit{Resource from } $\SetRes^{lim}$ \textit{at intermediary capacity.}
        \For{$i=1,\dots,\NumRes^{lim}$}
            \State $\mathcal{N}' \leftarrow (\mathcal{N}\setminus N_i)\cup \{N_i=\{0\}\}$
            \State $(K,I)\leftarrow$~\theMC-matrices$(\mathcal{N}',c,w,\NumTasks)$
            \State $k \leftarrow \gamma(i)$
            \Comment{\textit{Translates} $i$ \textit{to} $k$.}
            \For{$t=0,\dots,\Upper{k}-1$}
            \Comment{\textit{Checks all solutions with} $k$.}
            \If{$\Cost{k}(t) + K[\NumRes^{lim}][\NumTasks-t] < \TotalCost$}
                \State $\TotalCost \leftarrow\Cost{k}(t) + K[\NumRes^{lim}][\NumTasks-t]$
                \State $\Mapping \leftarrow$ Translate$(\gamma,\SetRes,\mathcal{N},w,I,t)$
                \State $\Map{k} \leftarrow t$
            \EndIf
            \EndFor
        \EndFor
        \State \textbf{return} $\Mapping$
    \end{algorithmic}
\end{algorithm}

\begin{algorithm}[h]
    \caption{Preparation for \theMC.}
    \label{algo:prep}
    \begin{algorithmic}[1]
        \Require Set of resources with upper limits $\SetRes^{lim}$ of size $\NumRes^{lim}$, set of upper limits $\SetUpper$, set of cost functions $\SetCosts$.
        \Ensure Set of disjoint classes of items $\mathcal{N} = \{N_1,\dots,N_{\NumRes^{lim}}\}$ with costs $c_{ij}$ and weights $w_{ij}$, $i=1,\dots,\NumRes^{lim}$, $j\in N_i$. Translation from disjoint classes to resources $\gamma$.
        \State $i \leftarrow 1$
        \ForAll{$r \in \SetRes^{lim}$}
            \State $\gamma(i) \leftarrow r$
            \State $N_i \leftarrow \{0, \Upper{r}\}$
            \Comment{\textit{Classes with 0 or $\Upper{r}$ tasks}.}
            \State $c_{i0} \leftarrow 0$~;~$c_{i\Upper{r}}\leftarrow \Cost{r}(\Upper{r})$
            \State $w_{i0} \leftarrow 0$~;~$w_{i\Upper{r}}\leftarrow \Upper{r}$
            \State $i \leftarrow i+1$
        \EndFor
        \State \textbf{return} $(\mathcal{N}, c, w, \gamma)$
    \end{algorithmic}
\end{algorithm}

\begin{algorithm}[h]
    \caption{Translation from \theMC~to a schedule.}
    \label{algo:trans}
    \begin{algorithmic}[1]
        \Require Translation from disjoint classes to resources $\gamma$.
        Set of resources $\SetRes$.
        Set of disjoint classes of items $\mathcal{N} = \{N_1,\dots,N_{\NumRes^{lim}}\}$ with weights $w_{ij}$, $i=1,\dots,\NumRes^{lim}$, $j\in N_i$. 
        Support matrix $I$ of dimensions $\NumRes^{lim}\times\NumTasks$.
        Knapsack capacity of interest $\NumTasks'$.
        \Ensure Partial schedule $\Mapping$.
        \State $\Map{i} \leftarrow 0,~\forall i \in \SetRes$
        \State $t \leftarrow \NumTasks'$
        \For{$i=\NumRes^{lim},\dots,1$}
            \Comment{\textit{Finds} $\Mapping$.}
            \State $j \leftarrow I[i][t]$~;~ $t \leftarrow t - w_{ij}$
            \State $\Map{\gamma(i)} \leftarrow j$
        \EndFor
        \State \textbf{return} $\Mapping$
    \end{algorithmic}
\end{algorithm}

Algorithm~\ref{algo:mardec} has a space bound in $O(\NumTasks\NumRes)$ due to its use of support matrices $K$ and $I$.
It requires~$O(\NumTasks\NumRes^2)$ operations. 
This number comes from the utilization of \theMC~in line~19.
As defined in Section~\ref{subsec:dp}, \theMC~requires $O(\NumTasks\sum_{i=1}^\NumRes |N_i|)$ operations.
In our case, each set of items $N$ contains only two items (i.e., scheduling zero or $\Upper{i}$ tasks to resource~$i$), so its complexity is in $O(\NumTasks\NumRes)$.
Given that \theMC~is computed at most $\NumRes+1$ times, the aforementioned complexity is achieved.

\subsubsection{Proof of optimality}

MarDec's optimality is demonstrated directly in Theorem~\ref{proof:mardec} based on the previous proofs for MarDecUn and \theMC.

\begin{theorem}\label{proof:mardec}
    The schedule $\Mapping$ computed by MarDec is optimal.
\end{theorem}
\begin{proof}
In the presence of decreasing marginal costs, Lemma~\ref{lemma:non} defines that an optimal schedule can be found in one of two scenarios: all remaining tasks are assigned to the same resource with minimum cost, or all tasks are assigned only to resources at maximum capacity and at most one resource at intermediary capacity (i.e., having two or more resources at intermediary capacity contradicts Lemma~\ref{lemma:non}).

MarDec computes a solution to the first scenario exactly as MarDecUn does. MarDecUn's solution is proved optimal in Theorem~\ref{proof:mardecun}.
For the second scenario, MarDec employs \theMC{} to compute all possible minimum-cost partial schedules using resources at maximum capacity.
These partial solutions are proved optimal in Theorem~\ref{proof:1:3}. 
All possible partial schedules are combined to all possible assignments of the remaining tasks to other resources with minimum cost.
By exhaustion, MarDec keeps the minimum-cost solution among every single possible solution that could provide a minimum cost in our two scenarios, therefore its schedule is optimal.
\end{proof}

\section{Concluding Remarks}\label{sec:conclusion}

\begin{table*}[!ht]
\centering
    \caption{Solutions with the smallest complexity for the variations of our scheduling problem. We assume $\NumTasks>\NumRes$.}
    \label{tab:dis}
    \begin{tabular}{c|c|ccc}

                     & \multirow{2}{*}{Arbitrary Costs} & \multicolumn{3}{c}{Marginal Costs}                     \\
                     & & Increasing & Constant & Decreasing \\
\toprule
Without upper limits & \theMC~--~$O(\NumTasks^2\NumRes)$ & MarIn~--~$\Theta(\NumTasks \log \NumRes)$ & MarDecUn~--~$\Theta(\NumRes)$ & MarDecUn~--~$\Theta(\NumRes)$ \\
With upper limits    & \theMC~--~$O(\NumTasks^2\NumRes)$ & MarIn~--~$\Theta(\NumTasks \log \NumRes)$ & MarCo~--~$\Theta(\NumRes \log \NumRes)$ & MarDec~--~$O(\NumTasks\NumRes^2)$
\end{tabular}
\end{table*}

In this paper, we considered the problem of minimizing the energy consumption of Federated Learning training on heterogeneous devices by controlling their workload distribution.
This problem is of growing interest, given the environmental costs of machine learning~\cite{schwartz2020,henderson2020} and FL systems~\cite{qiu2021}.
We have modeled this as the Minimal Cost FL Schedule problem, a total cost minimization problem with identical, independent, and atomic tasks that have to be assigned to heterogeneous resources with arbitrary cost functions.
In this process, we have defined a previously unexplored knapsack problem named Multiple-Choice Minimum-Cost Maximal Knapsack Packing Problem that generalizes our scheduling problem.
We have proposed an optimal solution for this knapsack problem based on dynamic programming and proved its optimality, thus solving the Minimal Cost FL Schedule problem too.

We have also explored scenarios with monotonically increasing cost functions with specific behaviors, and situations with and without upper limits.
Based on all optimal algorithms that we have proposed, Table~\ref{tab:dis} presents the algorithms that provide the lowest complexity for each scenario.
In all scenarios, but especially in scenarios with constant and decreasing marginal costs, the solution that minimizes the energy consumption may require that few resources to do most of the training.
In order to prevent over-representation from more energy-efficient devices~\cite{lim2020}, we recommend paying attention to the upper and lower limits set for all resources.

We would also like to emphasize that these algorithms are not only useful for energy conservation in FL systems: (I) they can also be used to minimize other kinds of costs (e.g., emissions of carbon dioxide or equivalents, financial costs), requiring only the cost estimates for different workload assignments; and (II) they can be applied to other problems that work with one-dimensional data partition~\cite{khaleghzadeh2018,khaleghzadeh2021}.

For the foreseeable future, we envision studies related to the application and adaptation of our algorithms.
First, we would like to conduct experiments in FL platforms to evaluate the impact of our algorithms compared to other solutions.
This impact should be measured in energy consumption, execution time, and accuracy of the model (or convergence speed).
Second, in terms of adaptation, new solutions may be required to handle dynamic changes in the system (e.g., changes in the cost behavior or loss of a device), to optimize the energy consumption of asynchronous FL systems, and to optimize FL systems that can offload parts of their computations to other Edge devices.

\section*{Acknowledgments}
The author would like to thank Dr. Amina Guermouche and Dr. Mihail Popov for their feedback on earlier versions of this manuscript.

\bibliographystyle{IEEEtran}
\bibliography{llpilla}

\end{document}